\documentclass{article}
\usepackage[letterpaper]{geometry}
\usepackage{pstricks,pifont}
\usepackage{tikz}
\usepackage{pgfplots,pgfplotstable}
\usepackage{amsmath}
\allowdisplaybreaks

\def\withcolors{1}
\def\withnotes{0}

\def\BibTeX{{\rm B\kern-.05em{\sc i\kern-.025em b}\kern-.08em
    T\kern-.1667em\lower.7ex\hbox{E}\kern-.125emX}}

\usepackage[utf8]{inputenc} 
\usepackage{hyperref} 
\usepackage{url} 
\usepackage{amsfonts} 
\usepackage{soul}
\usepackage{mleftright}
\usepackage{varwidth}
\usepackage[utf8]{inputenc} 
\usepackage[T1]{fontenc} 
\usepackage{hyperref}
\usepackage{url} 
\usepackage[noend]{algcompatible}
\usepackage{algorithm}
\usepackage[noend]{algpseudocode}

\usepackage[noadjust]{cite}
\usepackage{%
    amssymb,%
    amsthm,
    bbm,%
    etoolbox,%
    mathtools,%
    stmaryrd%
}

\usepackage{tikz,pgf,pgfplots,circuitikz}

\usepackage{caption}
\DeclareCaptionLabelFormat{algnonumber}{Algorithm}
\theoremstyle{plain} \newtheorem{thm}{Theorem}[section]
\newtheorem{lem}[thm]{Lemma} 
\newtheorem{cor}[thm]{Corollary}

\theoremstyle{definition} \newtheorem{defn}[thm]{Definition}

\theoremstyle{remark} \newtheorem{rem}{Remark} 
\newtheorem{fact}{Fact}

\definecolor{lightgray}{gray}{0.9}

\newcommand{\eq}[1]{\begin{align*}#1\end{align*}}

 \newcommand{\R}{\mathbb{R}}

\newcommand{\E}[1]{\mathbb{E}\left[#1\right]}
\newcommand{\EE}[2]{\mathbb{E}_{#1}\left[#2\right]}

\newcommand{\C}{\mathcal{C}}

\newcommand{\oO}{\mathcal{O}}

\newcommand{\uRoman}[1]{\uppercase\expandafter{\romannumeral#1}}

\DeclarePairedDelimiter\floor{\lfloor}{\rfloor}

\usepackage{subcaption}
\DeclareCaptionSubType*{algorithm}

\DeclareCaptionLabelFormat{alglabel}{Alg.~#2}

\ifnum\withcolors=1
   
   \newcommand{\mcolor}[1]{{\color{orange} {#1}}} 
   \newcommand{\vcolor}[1]{{\color{green} {#1}}} 
   \newcommand{\jcolor}[1]{{\color{pink} {#1}}} 
\else

\fi

\ifnum\withnotes=1
  \newcommand{\mnote}[1]{\par\mcolor{\textbf{PM: }\sf #1}} 
  \newcommand{\vnote}[1]{\par\vcolor{\textbf{VT: }\sf #1}} 
  \newcommand{\jnote}[1]{\par\jcolor{\textbf{JS: }\sf #1}}  
  
\else
  \newcommand{\mnote}[1]{}
  \newcommand{\vnote}[1]{}
  \newcommand{\jnote}[1]{}
  
\fi
\newcommand{\ignore}[1]{\leavevmode\unskip} 

\newcommand{\iid}{{\em i.i.d.}~}
\usepackage{soul}

\newcommand{\cE}{\mathcal{E}}

\newcommand{\SNR}{\mathtt{SNR}}

\newcommand{\CAS}{\mathtt{CAS}}
\newcommand{\ZSUCB}{\mathtt{UCB0}}
\newcommand{\UCB}{\mathtt{UCB}}
\newcommand{\UEUCB}{\mathtt{UE}\text{-}\mathtt{UCB}}
\newcommand{\UEUCBP}{\mathtt{UE}\text{-}\mathtt{UCB}\text{++}}

\algdef{SE}[SUBALG]{Indent}{EndIndent}{}{\algorithmicend\ }%
\algtext*{Indent}
\algtext*{EndIndent}
\input{tangocolors.sty} 
\definecolor{red1}{rgb}{0.4,0,0}

\DeclareMathOperator*{\argmin}{arg\,min}

\usepackage{lipsum}
\newcommand\blfootnote[1]{%
  \begingroup
  \renewcommand\thefootnote{}\footnote{#1}%
  \addtocounter{footnote}{-1}%
  \endgroup
}

\begin{document}

\title{Communication-Constrained  Bandits under  Additive Gaussian Noise}
\author{Prathamesh Mayekar\\ National University of Singapore\\\url{pratha22@nus.edu.sg} \and  Jonathan Scarlett\\ National University of Singapore\\ \url{scarlett@comp.nus.edu.sg} \and Vincent Y.F. Tan \\ National University of Singapore\\ \url{vtan@nus.edu.sg} }
\maketitle 








\begin{abstract}
We study a distributed stochastic multi-armed bandit where a client supplies the learner with communication-constrained feedback based on the rewards for the corresponding arm pulls. In our setup, the client must encode the rewards such that the second moment of the encoded rewards is no more than $P$, and this encoded reward is further corrupted by additive Gaussian noise of variance $\sigma^2$; the learner only has access to this corrupted reward. For this setting, we derive an information-theoretic lower bound of $\Omega\left(\sqrt{\frac{KT}{\SNR \wedge1}} \right)$ on the minimax regret of any scheme, where $\SNR\coloneqq \frac{P}{\sigma^2}$, and $K$ and $T$ are the number of arms and time horizon, respectively. Furthermore, we propose a multi-phase bandit algorithm, $\UEUCBP$, which matches this lower bound to a minor additive factor. $\UEUCBP$ performs uniform exploration in its initial phases and then utilizes the {\em upper confidence bound }(UCB) bandit algorithm in its final phase. An interesting feature of $\UEUCBP$ is that the coarser estimates of the mean rewards formed during a uniform exploration phase help to refine the encoding protocol in the next phase, leading to more accurate mean estimates of the rewards in the subsequent phase. This positive reinforcement cycle is critical to reducing the number of uniform exploration rounds and closely matching our lower bound.
\end{abstract}
\date{}
\blfootnote{An abridged version of this paper will appear in the proceedings of the International Conference on Machine Learning 2023.}
\newpage
\tableofcontents
\newpage

\section{Introduction}

The {\em stochastic multi-armed bandit} (SMAB) is a classic sequential decision-making problem with decades of research (see \cite{lattimore2020bandit} for a recent survey). The SMAB problem crucially captures the trade-off between exploring actions to find the optimal one or exploiting the action believed to be the best, but possibly sub optimal action, lying at the heart of many practical sequential decision-making problems. In its simplest form, the SMAB problem is a sequential game between the learner and the environment. The game is played over a finite number of rounds. In each round, the learner plays an arm from a given set of arms, and then the environment reveals a reward corresponding to the played arm. The reward for a particular arm is $iid$ across rounds, with the reward distribution unknown to the learner. The learner's goal is to maximize the expected cumulative reward over the entirety of the game.

In another direction, in modern distributed optimization frameworks such as Federated Learning \cite{konevcny2016federated}, a central {\em learner} trains a global machine learning model by leveraging updates from remote clients (typically, stochastic gradients computed on local data). However, a major bottleneck in such settings is the {\em uplink} communication cost -- the cost of sending remote client's updates to the learner. This has spurred a recent line of work that seeks to minimize the amount of communication by a remote client while still maintaining the accuracy of the global model  (see \cite{kairouz2019advances} for a recent survey).


An emerging abstraction to understand the trade-off between communication and accuracy in distributed optimization is to understand the one-client setting, where one remote client supplies the learner with gradient updates over a rate-limited channel (see, for instance,  \cite{mayekar2020ratq, acharya2021information, lin2020achieving, gandikota2019vqsgd}). A similar abstraction has been studied for distributed SMAB problems in two interesting recent works: \cite{hanna2022solving, mitra2022linear}. Our work is inspired by these two studies.

In particular, \cite{hanna2022solving} studies the  distributed SMAB problem where a client supplies the learner with reward feedback corresponding to its arm pull. To account for the uplink cost from the client to the learner, at the $t$th round, the client must quantize the reward feedback to $r_t$-bits, and only the quantized reward feedback is available to the learner for further decision-making. \cite{mitra2022linear} study a  linear distributed SMAB problem where at each round, the remote client supplies the learner with a $r$-bit quantized version of its estimate of the $d$-dimensional linear model parameter. These studies investigate similar questions: \\
\textbf{Q1a} \cite{hanna2022solving}: {\em  What is the minimum average bits $\frac{1}{T}\sum_{t=1}^{T}r_t$ needed to attain the regret of the centralized setting in the distributed bandit game?          }\\
\textbf{Q1b} \cite{mitra2022linear}: {\em  What is the minimum $r$ needed to attain the regret of the centralized setting in the distributed bandit game?          }\\
\cite{hanna2022solving} shows that $\oO(1)$ bits of average communication is sufficient to attain the regret of centralized setting, and \cite{mitra2022linear} shows that $r= \oO(d)$ is sufficient\footnote{The discrepancy between results is because of different setting: In \cite{hanna2022solving}, the scaler reward feedback needs to be compressed, whereas in \cite{mitra2022linear} the $d$-dimensional linear model parameter needs to be compressed.} to attain the regret of the centralized setting. 

In many practical distributed learning scenarios, the updates from the client to the learner are sent over a noisy wireless network. To model this, recent works on  distributed optimization  have considered communication over noisy channels; see, for instance, \cite{jha2022fundamental, amiri2020machine} and the references therein. Unfortunately, the results for distributed optimization over noisy channels do not apply to the distributed SMAB problem over noisy channels. Moreover, the model for communication constraints in the recent works on  distributed SMAB
does not accurately capture communication over noisy channels.
To bridge this gap, we study a distributed SMAB problem where the remote client codes the reward feedback to ensure that the encoded reward has a second moment less than $P$, and then the client transmits the encoded reward over an (AWGN) channel with noise variance $\sigma^2$ (this is equivalent to corrupting the encoded reward with an additive Gaussian noise of mean zero and variance $\sigma^2$); only the channel's output is available to the learner for pulling the arm at the next round. 
 We define the signal-to-noise ratio in our problem as \[\SNR \coloneqq \frac{P}{\sigma^2}.\] 
 
{We remark that the restriction of the second moment of the encoded messages to $P$ and transmission over an AWGN channel is a ubiquitous communication model in information theory and communication and dates back to Shannon's seminal work; see, for instance, \cite{Shannon}, \cite[Chapter 9]{CovTho06}, and \cite[Chapter  8]{tse2005fundamentals}. The second moment constraint   is used to model that most edge devices are restricted to transmitting up to a particular transmission power due to physical constraints. The AWGN channel is often used as a first approximation for various communication settings.
}

 Coming back to comparisons with \cite{hanna2022solving} and  \cite{mitra2021linear}, the counterpart of \textbf{Q1} in our setting is the following question:\\
\textbf{Q2:} {\em  What is the minimum $\SNR$ needed to attain the regret of the centralized setting in the distributed bandit game over an AWGN channel?}\\
In this work, we resolve the following more general question upto multiplicative $\log$  and minor additive factors:\\
\textbf{Q3: }{\em  How does the minimax regret scale as a function of $\SNR$?}


\subsection{Our Contributions}

We show that any bandit algorithm  employed by the learner and any encoding strategy employed by the client must incur a worst-case expected cumulative regret of at least\footnote{ $a\wedge b \coloneqq \min\{a, b\}$; $a \vee b \coloneqq \max\{a, b\}$.}
\[ \Omega \left(\sqrt{\frac{KT}{\SNR \wedge1}} +KB \right),\]
where $K$ is the number of bandit arms, $T$ is time horizon, and $B$ is upper bound on the $L_2$-norm\footnote{The $L_2$ norm of a random variable $X$ is $\sqrt{\E{X^2}}.$} of rewards. A technical difficulty in deriving this lower bound is to account for the loss of information due to transmission over the  AWGN channel. We overcome this difficulty by leveraging classic information-theoretic results for transmission over an AWGN channel and combining these techniques with the standard minimax lower bound proof for SMAB. An interesting consequence of our proof technique is that, up to constant factors, we can recover the lower bound in \cite{hanna2022solving} with an arguably simpler proof; see Appendix for details.
%

Unlike \cite{hanna2022solving} and \cite{mitra2022linear}, we don't quantize our updates. In our setup, quantization and transmitting over an AWGN channel would lead to decoding errors at every round, which can pile up and result in linear regret unless $\SNR$ grows with the time horizon $T$; such a scheme would not allow us to resolve $\textbf{Q3}.$ Instead, we use a {\em center and scale} encoding protocol $\CAS$, where the centering and scaling factors are learned over time. We also design a new bandit algorithm $\UEUCBP$  which, along with the encoding protocol $\CAS$, bounds the worst-case expected cumulative regret by a factor of\footnote{$\log(\cdot)$ refers to the logarithm to the base 2; $\ln$ refers to the logarithm to the base $e$.}
\[\oO\left(  \sqrt{\left(\frac{1}{\SNR}+1\right) \cdot KT \ln T}  +\frac{KB \log B}{\SNR \wedge 1} \right). \] Thus, we resolve \textbf{Q3} up to a multiplicative factor of $\sqrt{\ln T}$ and an additive factor of $\frac{KB \log B}{\SNR \wedge 1}$.

$\UEUCBP$ performs uniform exploration for multiple initial sub-phases and then employs the classic upper confidence bound $\UCB$ bandit algorithm. The multiple initial uniform exploration sub-phases allow us to learn $\CAS$'s optimal centering and scaling factors. An interesting feature of $\UEUCBP$ is that we start with coarser mean reward estimates for all arms in a given sub-phase, which helps us refine the communication scheme for the subsequent sub-phases. The refined communication scheme, in turn, leads to better mean rewards estimates in the subsequent sub-phase. Due to this positive reinforcement cycle, the variance of the mean reward estimates decays exponentially with respect to the number of uniform exploration sub-phases. As a result, we can learn an almost optimal encoding protocol after relatively few rounds of uniform exploration. 

We also design a private\footnote{As opposed to the popular privacy notion of differential privacy \cite{dwork2014algorithmic}, we emphasize that here we refer to a much more rudimentary privacy notion, where the learner does not send information to the remote clients at any of the rounds.} scheme where the learner does not send any information to the client. We show that for such a scheme, the worst-case expected cumulative regret can be bounded by
\[ \oO\left( \sqrt{ \frac{B^2}{\SNR}+1} \cdot \sqrt{KT \ln T}  + KB \right).\]
This simple scheme nearly matches our lower bound when $B=\oO(1).$



\section{Setup and Preliminaries}

\begin{figure*}
\centering
\begin{subfigure}[t]{\textwidth}
\centering
\begin{tikzpicture}[scale=0.7, transform shape,
    pre/.style={=stealth',semithick},
    post/.style={->,shorten >=0.25pt,>=stealth',semithick},
dimarrow/.style={->, >=latex, line width=1pt},
normline/.style={-, line width=1pt}
    ]
\clip (-2.1,6) rectangle  (20.6,-0.6) ;

\draw[black!30!,thick, fill = tacream ] (11.9, 4.6) rectangle( 18.4, -0.4);
\node[align=center] at (15.3, 0.25) {\color{black}{\LARGE{\color{red1} \textbf{Learner}}}}; 

\draw[black!30!,thick, fill = tacream ] (-1.3, 4.6) rectangle(4.6, -0.40);
\node[align=center] at (1.95,0.25) {\LARGE{\color{red1}\textbf{Client}}};

\draw[dimarrow, orange] (15,3.5) to (1,3.5);
\draw[dimarrow, orange] (14.25,3.5) to (14.25, 2);
\node[align=center] at (13.3 ,4) {\Large{$A_t, S_t, D_t$}};
\node[align=center] at (13.2, 2.85) {\Large{$D_t$}};

\draw[dimarrow, orange] (3.25,3.5) to (3.25, 2);
\node[align=center] at (2.4 , 2.85) {\Large{$S_t$}};
\node[align=center] at (2 , 4) {\Large{$A_t$}};
\draw[dimarrow, orange] (15,1.5) to (16.05, 1.5);
\node[align=center] at (16.2, 0.9) 
{\Large{$D_t(Y_t) $}};

\draw[dimarrow, orange] (1.,1.5) to (2.6,1.5);
\node[align=center] at (1.8, 0.9) 
{\Large{$X_{t, A_t}$}};
\draw[dimarrow, orange] (4.05,1.5) to (8,1.5);
\draw[dimarrow, orange] (10.5,1.5) to (13.5, 1.5);
\node[align=center] at (6.23,0.9) 
{\Large{$C_t(X_{t, A_t}, S_t)$}};

\draw[tagreen, thick, fill = tagreen] (8, 2.5) rectangle(10.5, 0.5);
\node[align=center] at (11.3, 0.9) 
{\Large{$Y_t $}};
\node[align=center] at (9.25, 1.5) 
{\Large{AWGN}};
\node[align=center] at (8.75, 1.5) 
{};

\draw[gray!20, thick, fill = gray!20 ] (15, 4) rectangle(16.5, 3);
\node[align=center] at (15.75, 3.5) {\Large{$\pi$}};
\draw[dimarrow, orange] (18, 3.5) to (16.5, 3.5);
\node[align=center] at (17.5, 3) {\Large{$H_{t-1}$}};

\draw[gray!20 , thick, fill = gray!20 ] (13.5, 2) rectangle(15, 1);
\node[align=center] at (14.25, 1.5) 
{\Large{$D_t$}};

\draw[gray!20 , thick, fill = gray!20 ] (2.5, 2) rectangle(4, 1);
\node[align=center] at (3.25, 1.5) 
{\Large{$C_t$}};
\end{tikzpicture}

\end{subfigure}
\caption{The Distributed Bandit Game over an AWGN channel}\label{f:setup}
\end{figure*}

We now proceed to describe our setup; see Figure \ref{f:setup} for a pictorial description.

\paragraph{The Distributed SMAB Problem  over an AWGN channel:}

\begin{enumerate}
\item At round $t \in [T]$,\footnote{For an integer $n$, $[n]\coloneqq\{1, \ldots, n\}$.} the learner asks a client  to pull an arm $A_t$. The learner also sends some side information $S_t$ based on the history of the game until time $t-1$ to aid the communication at the client.

    \item Upon receiving this information, the client pulls arm  $A_t$ and gets a reward $X_{t, A_t}$ corresponding to the arm $A_t$. The client then encodes this reward using an encoding function 
    $ C_t \colon \R \times \mathcal{S}  \to \R$ to form $C_t(X_{t, A_t}, S_t),$  where $\mathcal{S}$ is the domain of the information $S_t$ sent by the learner.
    The encoding  function must satisfy the power constraint 
    \begin{align}\label{eq:Power_Constraint}
\E{C_t(X_{t, A_t}, S_t)^2} \leq P
\end{align} 
for some $P >0.$

\item The client sends the encoded reward over the AWGN channel. The output of the channel $Y_t$ is received by the learner, where 
\begin{align}\label{eq:Channel_Output}
Y_t = C_t(X_{t, A_t}, S_t) + Z_t
\end{align}
and $\{Z_t\}_{t=1}^\infty$ is sequence of \iid Gaussian random variables with mean $0$ and variance $\sigma^2.$
\item The learner decodes $Y_t$ using a decoding function $D_t \colon \R \times  \mathcal{S}  \to \R$ to form a reward estimate $D_t(Y_t, S_t).$ The learner then uses a bandit algorithm $\pi$ to decide on arm $A_{t+1}$, the side information $S_{t+1}$, and decoding function $D_{t+1}$ using history of the game up to round $t$, $H_t \coloneqq (A_1, Y_1, \ldots, A_{t}, Y_{t}).$ 
\end{enumerate}

\begin{rem}
In most distributed learning settings, the communication from the learner to client (downlink communication) is not expensive. Therefore, we allow for schemes where the learner can, in principle, send entire history of the game up to a particular round as side information for the next round. However, our proposed scheme will send only a single real number as side information at each round, which is more desirable in practical settings.
\end{rem}

\begin{rem}
The client is `memoryless' in our setup; the client cannot use rewards from previous arm pulls to encode the current reward. This is motivated by the fact that, as observed in \cite{hanna2022solving}, in a typical distributed sequential decision-making problem, the client observing the reward is a low memory sensor. Additionally, in some applications, the client at any given round may differ from those at the previous rounds. 
\end{rem}

Denote by $\Pi_T$ be the set of all possible $T$-round  bandit algorithms as described above and  denote by $\mathcal{C}$ the set of all possible encoding strategies satisfying \eqref{eq:Power_Constraint}. We note that $\Pi_T$ contains randomized algorithms, so our lower bounds hold for such algorithms, too. However, we will use a deterministic algorithm for our upper bounds.

\paragraph{Reward Distributions:}
We make the following standard assumptions on the reward distributions. We assume that the reward sequence $\{X_{t, i}\}_{t >1}$ is \iid for all arms $i$ in $[K]$ and the rewards for different arms are independent of each other.  Denote by $\mu_i$ the mean reward of arm $i$. We assume that $X_{1, i}-\mu_i$ is subgaussian with variance factor $1$,\footnote{{The variance factor $1$ is without loss of generality; all our regret bounds scale linearly with this factor.}} where a subgaussian random variable  with a variance factor $\sigma$ is as defined below.
\begin{defn}\cite{boucheron2013concentration}
\textit{We say a random variable $X$ is subgaussian  with a variance factor $\alpha^2,$ if  for all $\lambda \in \R,$ we have \[ \ln [\exp( \lambda X)] \leq \frac{\lambda^2 \alpha^2}{2}.\] 
Furthermore, for a subgaussian random variable with variance factor $\alpha^2,$ we have $\E{X}=0$ and $\E{X^2} \leq \alpha^2.$}
\end{defn}
{Finally, for  $B\geq 1$,\footnote{Since $X_{1, i}-\mu_i$  is subgaussian with variance factor $1,$  its  second moment is bounded by $1$, so we may assume that $B \geq 1$.}  we assume for all $i \in [K]$, that
\begin{align}\label{eq:Bounded_Mean}
    \E{X_{1, i}^2} \leq B^2.
\end{align}
}
\mnote{The subscript has $1$ instead of $t$ because we have already mentioned the process is $iid$.}



Denote by $\mathcal{P}$ the set of  all subgaussian distributions with variance factor $1$ whose second moment satisfies \eqref{eq:Bounded_Mean}. Then, $\mathcal{E} \coloneqq \mathcal{P}^K$ describes the set of all possible bandit instances considered in our setup.

\paragraph{Regret:}
We define the regret for a bandit instance $\nu \in \mathcal{E}$, when a bandit algorithm $\pi \in \Pi$ and  encoding strategies $C_1^T := (C_{t})_{t \in [T]} \in \mathcal{C}^T$ are employed, as
 \[R_T(\nu, \pi, C_1^T) \coloneqq T\mu_{i^\ast} -\E{\sum_{t=1}^{T}X_{t,A_t}},\]
 where $i^\ast$ is the arm with the largest mean for the bandit instance $\nu$. We are interested in characterizing the minimax regret
\begin{align}\label{e:minmax_regret}
 R^\ast_T(\SNR) \coloneqq \inf_{\pi \in \Pi_T} \inf_{ C_1^T \in \mathcal{C}^T} \sup_{\nu \in \mathcal{E}}R_T(\nu, \pi, C_1^T)
 \end{align}
as function of $\SNR, K, B,$ and $T$.

\paragraph{$\UCB$ bandit algorithm:} We now describe the classical {\em Upper Confidence Bound} ($\UCB$) bandit algorithm \cite{auer2002finite}. We will build on  the $\UCB$ algorithm to come up with our algorithm. Recall that at each round $\UCB$ plays the arm with the largest upper confidence bound. For subgaussian rewards with variance factor $\alpha^2$, the upper confidence bound for an arm $k \in [K]$ after $t$ rounds is defined as
\[ 
\UCB_{k}(t; \eta) = 
\begin{cases}
\infty \quad \text{if} \quad N_k(t)=0\\
\hat{\mu}_k(t)+ \sqrt{\frac{4 \eta \ln T}{N_k(t)}} \quad \text{otherwise},
\end{cases}
\]
where $N_k(t)$ is the number of times the arm $k$ has been played up to round $t$, and $\hat{\mu}_k(t)$ is the empirical mean estimate of arm $k$. It is not difficult to prove the following regret guarantees for the $\UCB$ algorithm  (see, for instance, \cite[Theorem 7.2]{lattimore2020bandit}). 

\begin{lem}\label{l:UCB}
Suppose we have a classical bandit game over $T$ rounds and $K$ arms with the following reward feedback: (i), the reward feedback corresponding to all arms is subgaussian with a variance factor of $\alpha^2$; (ii),
 the second moment of the reward feedback is at most $B$.
Then, the regret of the $\UCB$ algorithm  with parameter $\eta=\alpha^2$ satisfies
\[R_T \leq 8\sqrt{\alpha^2 KT \ln T} +6KB.\]

\end{lem}

\section{Information-Theoretic Lower Bound}

Our lower bound is stated as follows.
\begin{thm}\label{t:lb}
For universal   constants $c_1$ and $c_2$,  where $0 < c_1 <1$ and $c_2 > 0,$ we have
\eq{R^\ast_T(\SNR) \geq c_1 \bigg( \sqrt{KT} \cdot \frac{1}{\sqrt{\frac{1}{2} \log (1+ \SNR) \wedge 1}} \wedge c_2 T     + KB \bigg)  }
for all $T \geq  \frac{c_2 K}{\frac{1}{2} \log (1+ \SNR) \wedge 1}$.
\end{thm}
The proof of this lower bound adopts a similar high-level approach to the lower bound proofs in the standard setting (see, for instance, \cite[Chapter 15]{lattimore2020bandit}). However, a crucial difference from the standard proofs is that we now have to account for the loss in information due to the rewards being sent over the Gaussian channel. We defer the proof to the Appendix.

\begin{rem}
{In information theory, the term $\frac{1}{2} \log (1+ \SNR)$  is  the capacity of the Gaussian channel. It refers to the maximum rate at which information can be transmitted with arbitrarily small probability of error; see 
\cite{CovTho06} for details.}
\end{rem}

The following corollary of Theorem \ref{t:lb} will be useful to derive matching upper bounds.

\begin{cor}
For universal   constants $c_1$ and $c_2$,  where $0 < c_1 <1$ and $c_2 > 0,$ we have
\[R^\ast_T(\SNR) \geq c_1 \left(\sqrt{KT} \cdot \frac{1}{ \sqrt{\SNR  \wedge 1}} + KB \right)\]
for all $T \geq  \frac{c_2 K}{\frac{1}{2} \log (1+ \SNR) \wedge 1}$.
\end{cor}
\begin{proof}
The proof simply follows by combining Theorem \ref{t:lb} with the fact that $\ln(1+x) \leq x$ for $ x \geq 0$.
\end{proof}

We note while $\SNR \wedge 1$ and $\frac{1}{2} \log(1+\SNR) \wedge 1$ may appear quite different at first glance, they are, in fact, the same up to constant factors, since  $\frac{1}{2} \log(1+\SNR)$
 is both upper  and lower bounded by  a constant times $\SNR$ for $\SNR \in [0, 3],$ whereas for $\SNR > 3$ both terms take the value $1.$
\section{Bandit Algorithms over an AWGN Channel}

\subsection{Our Encoding Algorithm: $\CAS$}
 A crucial component of all our schemes is an encoding strategy  we refer to as { \em centered amplitude scaling} ($\CAS$). $\CAS$ simply centers the input $X$ around side information $S$ and multiplies it by $\theta$ to get \[C_{\CAS}(X, S; \theta):= \theta (X-S),\]
which is transmitted over the AWGN channel. If the input  and side information are such that $\E{(X -S)^2}\leq \frac{P^2}{\theta^2},$ then $C(X, S)$ satisfies the power constraint  \eqref{eq:Power_Constraint}.

We now build up to our main algorithm in the following two subsections.
\subsection{Warm-up: $\ZSUCB$ bandit algorithm}

We begin by describing a simple  warm-up which uses  $\CAS$ with parameter $\theta=\frac{\sqrt{P}}{B}$ as the encoding algorithm at each round, and  the $\ZSUCB$ bandit algorithm described below.

The $\ZSUCB$ bandit algorithm is a standard bandit algorithm  that  sends no side information at all rounds (equivalent to sending zero as the side information at all rounds). Additionally, $\ZSUCB$ simply scales channel output $Y_t$ by the inverse of the scaling parameter used by $\CAS$,  $\frac{1}{\theta}=\frac{B}{\sqrt{P}},$ to get 
\[\tilde{X}_{t, A(t)}=\frac{B}{\sqrt{P}} Y_t,\] which it uses as the reward estimate for round $t.$ $\ZSUCB$  along with its interaction with $\CAS$ is described in Algorithm \ref{a:ZSUCB}.

\setcounter{figure}{0}
\begin{figure}[t]

\centering
\begin{tikzpicture}[scale=1, every node/.style={scale=1}]
\node[draw,text width= 12cm, text height=,] {%
\begin{varwidth}{\linewidth}

\algrenewcommand\algorithmicindent{0.7em}

\begin{algorithmic}[1]
\Statex   \hspace{-0.6cm} \textbf{Parameters:} $\eta=\frac{B^2}{\SNR}+1$, $\theta=\frac{\sqrt{P}}{{B}}$
\For{$t \in [T]$}
\State {\color{red1} \textbf{At Learner}}
\Indent \State    $\displaystyle{A_t = \arg\max_{k\in K} \UCB_k\left(t-1; \eta\right)}$ and $S_t=0$ 
\State Send $A_t,$ $S_t$ to Client
\EndIndent

\State {\color{teal} \textbf{At Client}}
\Indent
\State Play arm $A_t$ and observe $X_{t, A_t}$
\State Send $C_{\CAS}( X_{t, A_t}, 0; \theta)$ over the AWGN  channel
\EndIndent

\State {\color{red1} \textbf{At Learner}}
\Indent
\State Observe AWGN Output $Y_{t}=C_{\CAS}( X_{t, A_t}, 0)+Z_t$
\State Use $\tilde{X}_{t, A_t}=\frac{Y_t}{\theta} $ as the $t$th round reward 
\State Update UCB index for all arms
\EndIndent
\EndFor
\end{algorithmic}  
\end{varwidth}};
 \end{tikzpicture}
 \renewcommand{\figurename}{Algorithm}
 \caption{$\ZSUCB$ bandit algorithm along with $\CAS$ encoding algorithm}\label{a:ZSUCB}
\vspace{-0.5cm}  
 \end{figure}

\begin{thm}\label{t:UCB0}
For all bandit instances $\nu \in \cE$, the $\ZSUCB$ algorithm  along with $\CAS$ encoding   satisfies the power constraint \eqref{eq:Power_Constraint} at all rounds and has the following regret guarantee:
\[R_T(\nu, \UCB0, \CAS)  \leq 8 \sqrt{ \frac{B^2}{\SNR}+1} \cdot \sqrt{KT \ln T} + 6 KB.\]
\end{thm}
\begin{proof}
Using \eqref{eq:Bounded_Mean}, we have that $\E{\big(\frac{\sqrt{P}}{B}X\big)^2}\leq P.$
We will show that the reward estimates $\tilde{X}_{t, A_t}$ used by the UCB algorithm are unbiased and subgaussian with a variance factor of $\frac{B^2}{\SNR}+1$. 
First, notice that $\tilde{X}_{t, A_t}=\frac{B Y_t}{\sqrt{P}}$ and $Y_t=\frac{\sqrt{P}{X}_{t, A_t}}{B}+ Z_t.$ Combining these, we get
\[\tilde{X}_{t, A_t}= X_{t, A_t} + \frac{B Z_t}{\sqrt{P}}.\]
 Note that $X_{t, A_t}$  and $Z_t$ are independent and subgaussian with variance factors $1$ and $\sigma^2$, respectively. Therefore, the centered reward estimate, $\tilde{X}_{t, A_t} -\mu_{A_t},$ is subgaussian with variance factor $\frac{B^2}{\SNR}+1$. 
 
 Hence, this algorithm reduces our problem to playing a bandit game with subgaussian rewards that have a variance factor of $\alpha^2=\frac{B^2}{\SNR}+1.$ Since the learner uses the UCB algorithm with parameter $\eta=\alpha^2,$ we can bound the regret by using Lemma \ref{l:UCB}.
\end{proof}

\begin{rem}\label{r:privacy}
The regret bound of $\ZSUCB$ matches our lower bound when $B=\oO(1).$
However, for larger values of $B,$ the regret bound of $\ZSUCB$ is far from our lower bound. While this is a significant weakness of $\ZSUCB$, it is still appealing in settings where the learner does not want to send any side information to the client due to privacy concerns.
\end{rem}

\subsection{ $\UEUCB$ bandit algorithm}
 The regret bound in Theorem \ref{t:UCB0} is off by a multiplicative factor of ${\oO\big(B\big)}$ from our lower bound. This gap results from our encoding algorithm, which leads to subgaussian rewards with a variance factor of $\frac{B^2}{\SNR}$. On the other hand, to match  the lower bound in Theorem \ref{t:lb}, one requires a strategy  such that the encoded rewards are subgaussian with a variance factor of $\oO\left( \frac{1}{\SNR}\right).$ We take the first step in arriving at such an algorithm in our second scheme, in which we combine a two-phase bandit algorithm, { \em Uniform Exploration-Upper Confidence Bound } ($\UEUCB$), with the $\CAS$ encoding algorithm. 
 
 In the first phase, $\UEUCB$ performs uniform exploration, and in the second phase, $\UEUCB$ uses the $\UCB$ algorithm to select the arm. We will now describe in detail the $\UEUCB$ algorithm and the encoding strategy at the client.
 
 \paragraph{Phase 1:}
In  the first phase, the learner uniformly explores all arms, and sends no side information to the clients. The communication strategy in this phase is the same as in the previous scheme. In particular,  we use $\CAS$ with parameter $\theta_1=\frac{\sqrt{P}}{B}$ for encoding at the client. For decoding at the learner,  $\UEUCB$ scales the channel output $Y_t$ by $1/\theta_1$
 to get
 \[\tilde{X}_{t, A(t)}=\frac{B}{\sqrt{P}} Y_t.\] We run the phase one of $\UEUCB$ for $K \tau$ rounds, where\footnote{For simplicity of notation, we assume that $\frac{B^2}{\SNR},$ $\frac{1}{\SNR},$ and $\log B$ are integers throughout the paper. If not, the rounding only slightly impacts the constant factors in our results.} $\tau \coloneqq \frac{B^2}{\SNR}+1$.
 
 \paragraph{Phase 2:}
  For all $k \in [K],$ denote by $\hat{\mu}_k$  the mean estimates formed for all arms at the end of the first phase. Then, in the second phase, the learner runs the classic $\UCB$ algorithm  treating $\tilde{X}_{t, A_t}$  as the reward to decide on an arm $A_t$. It then informs the client about arm $A_t$ and the mean estimate   formed for this arm after phase one, $\hat{\mu}_{ A_t}$.
  
  On the client's side, 
  the reward  corresponding to $A_t$ is encoded by using $\CAS$ with parameter $\theta_2=\sqrt{\frac{P}{2}}$ to get
 \[
 C_{\CAS}(X_{t, A_t}, \mu_{A_t}; \theta_2):= \sqrt{\frac{P}{2}} (X_{t, A_t}-\hat{\mu}_{A_t}).
 \]

 The decoding at the learner is simply done by scaling $Y_t$ by $\frac{1}{\theta_2}$  and adding the mean estimate $\hat{\mu}_{A_t}$ to get
 \[ 
 \tilde{X}_{t, A_t}=\sqrt{\frac{2}{P}}Y_t+\hat{\mu}_{A_t}. 
 \]

$\UEUCB$ along with its interaction with $\CAS$ is described in Algorithm \ref{a:UEUCB}.

\begin{rem}
 {We will soon show in the proof of Theorem \ref{t:UEUCB} that $\E{\big(X_{t, A_t}-\hat{\mu}_{A_t}\big)^2} \leq 2.$ Thus our choice of $\theta_2$ ensures that the power constraint is satisfied. Moreover, for $B\geq 2$, $\theta_2$ is greater than $\theta_1$, which essentially means that the client can scale the centered reward in the second phase by a bigger factor as compared to the scaling factor used in Algorithm \ref{a:ZSUCB}. In both algorithms, since the learner multiplies $Y_t$ by the inverse of the scaling factor used at the client,  the greater value of scaling factor results in smaller variances for the decoded rewards than in Algorithm \ref{a:ZSUCB}.}
 \end{rem}
 
 \mnote{Using big for the expectation in the first line makes the superscript 2 to stand out. So keeping it as is.}
 
 \renewcommand{\figurename}{Algorithm}
\begin{figure}[t]

\centering
\begin{tikzpicture} [scale=1, every node/.style={scale=1}]
\node[draw,text width= 12cm, text height=,] {%
\begin{varwidth}{\linewidth}
            
            \algrenewcommand\algorithmicindent{0.7em}
\begin{algorithmic}[1]

\Statex  \hspace{-0.6cm} \textbf{Parameters:}  $\theta_1=\frac{\sqrt{P}}{{B}}$, $\tau=\frac{B^2}{\SNR}+1$
  \Statex \hspace{-0.5cm}{\Large \color{blue}Phase $1$}
\For{ $k \in [K]$} 
 \For{ $t \in \{(k-1)\tau+1, \ldots, K\tau\}$}
 \State {\textbf{ \color{red1}At Learner}}
 \Indent
\State   $A_t = k$ and $S_t=0$
\State Send $A_t, S_t$ to Client
\EndIndent
\State { \textbf{\color{teal}At Client}} 
\Indent
\State Play arm $A_t$ and observe $X_{t, A_t}$
\State  Send $C_{\CAS}( X_{t, A_t}, 0; \theta_1)$ over the AWGN channel
\EndIndent

\State {\textbf{\color{red1}At Learner}} 
\Indent
\State Observe AWGN Output $Y_{t}=C_{\CAS}( X_{t, A_t}, 0; \theta_1)+Z_t$
\State Use $\tilde{X}_{t, A_t}=\frac{Y_t}{\theta_1} $ as the $t$th round reward 

\EndIndent

\EndFor

\State {\textbf{\color{red1}At Learner}}     
\Indent
\State  Form  a mean estimate for arm $k$ 
\[\displaystyle{\hspace{1cm}\hat{\mu}_{ k} = \frac{1}{\tau}\cdot \sum_{t = (k-1)\tau+1}^{ k\tau} \tilde{X}_{t, k}}\]
\EndIndent
\EndFor

\Statex~
\Statex  \hspace{-0.6cm} \textbf{Parameters:} $B_{2}^2= 2$, $\eta_2= \frac{2}{\SNR} +1 $, $\theta_2=\sqrt{\frac{P}{2} }$
\Statex \hspace{-0.5cm}{\Large \color{blue}Phase $2$}
\For{$t \in \{K\tau+1, \ldots, T\}$}
\State{\textbf{\color{red1} At Learner}}
\Indent
\State   $\displaystyle{A_t = \arg\max_{k\in K} \UCB_k\left(t-1; \eta_2\right)}$ and $S_t=\hat{\mu}_{A_t}$
\State Send $A_t, S_t$ to Client
\EndIndent
\State{\textbf{\color{teal}At Client}} 
\Indent
\State Play arm $A_t$ and observe $X_{t, A_t}$
\State  Send $C_{\CAS}( X_{t, A_t}, S_t; \theta_2)$ over the AWGN channel
\EndIndent

\State {\textbf{\color{red1}At Learner}} 
\Indent
\State Observe AWGN Output $Y_{t}=C_{\CAS}( X_{t, A_t}, S_t; \theta_2)+Z_t$
\State Use $\tilde{X}_{t, A_t}=\frac{Y_t}{\theta_2} +S_t$ as the $t$th round reward 
\State Update UCB index for all arms
\EndIndent

\EndFor

\end{algorithmic}  
\end{varwidth}};
 \end{tikzpicture}
 \renewcommand{\figurename}{Algorithm}
 \caption{$\UEUCB$ Algorithm along with the $\CAS$ encoding algorithm}\label{a:UEUCB}
\vspace{-0.5cm}  
 \end{figure}

  \begin{thm}\label{t:UEUCB}
For all bandit instances $\nu \in \cE$, the $\UEUCB$  algorithm  along with $\CAS$ encoding  as described in Algorithm \eqref{a:UEUCB} satisfies the power constraint \eqref{eq:Power_Constraint} at all rounds and has the following regret guarantee:
\vspace{-0.5cm}
\eq{
R_T(\nu, \UEUCB, \CAS)  \leq  \frac{2KB^3}{\SNR} +8KB  +  8\sqrt{\frac{2}{\SNR}+1} \sqrt{KT \ln T}.}
\end{thm}
\begin{proof}

In the first phase, we incur  regret linear regret in the time horizon. Specifically, since the time horizon for the first phase is $K\tau=K (B^2/\SNR+1)$ and the gap between the mean of the optimal arm and the suboptimal arm can at most be $2B$, we have that the regret incurred in the first phase is $(KB^2/\SNR +K) \cdot 2B.$ Moreover, since the encoding of rewards on the client's side and decoding of the reward on the learner's side in the first phase is same as in  Algorithm \ref{a:ZSUCB}, by the same arguments, the power constraint is satisfied, and for all $t \in \{(k-1)\tau+1, \ldots,  k\tau\},$
\begin{equation}\label{e:phase1_var}
  \E{(\tilde{X}_{t, k}-\mu_k)^2} \leq \frac{B^2}{\SNR}+1.
\end{equation}

For the second phase, we  will first show that the power constraint is satisfied. Towards proving this, we will bound the variance of the mean estimates formed after the first phase. We have from the description of Algorithm \ref{a:UEUCB} that for all $k\in K,$
\eq{
\E{(\hat{\mu}_k-\mu_k)^2} &= \frac{1}{\tau^2} \sum_{t = (k-1)\tau+1}^{ k\tau} \E{(\tilde{X}_{t, k}-\mu)^2}\leq 1,
}
where the inequality follows from  \eqref{e:phase1_var} and the fact that $\tau=\frac{B^2}{\SNR}+1.$

Now,  note that for any round $t$ in phase 2,  we have
\eq{&\E{(X_{t,A_t} -\hat{\mu}_{A_t})^2} \\&=\E{(X_{t, A_t} -\mu_{A_t})^2}+\E{(\hat{\mu}_{A_t}-\mu_{A_t})^2} \leq 2,
}
where the equality follows from the independence of $X_{t,A_t}$  and $\hat{\mu}_{A_t},$ and   the inequality follows from the fact that both $(X_{t,A_t} -\mu_{A_t})$ and $(\hat{\mu}_{A_t}-\mu_{A_t})$ are subgaussian with a variance factor of $1$.

Therefore, scaling the centered reward $X_{t,k} -\hat{\mu}_k$ by $\frac{\sqrt{P}}{\sqrt{2}}$ satisfies the power constraint.

 The reward estimate used by the learner in phase 2 is \eq{\tilde{X}_{t, A_t}&=\frac{\sqrt{2}Y_t}{\sqrt{P}}+\hat{\mu}_{A_t}\\
&=X_{t, A_t} + \frac{\sqrt{2}Z_t}{\sqrt{P}}.
}
 This implies that the centered reward estimate, $\tilde{X}_{t, A_t} -\mu_{A_t}$, is  subgaussian with a variance factor of  $2/\SNR +1$ in the second phase. The proof is complete by upper bounding the regret in the second phase by using Lemma \ref{l:UCB} and summing the regret in the two phases.
\end{proof}

 \subsection{$\UEUCBP$ bandit algorithm}

 Algorithm \ref{a:UEUCB} improves over Algorithm \ref{a:ZSUCB} in the sense that the $T$-dependent term in its regret upper bound does not have the multiplicative $\oO(B)$ dependence. This term matches the lower bound in Theorem  \ref{t:lb} up to logarithmic factors in $T$. However, this comes at the cost of an additional additive term of $\oO(\frac{KB^3}{\SNR})$ in the regret bound of Algorithm \ref{a:UEUCB}. The strong dependence on $B$ in this additive term means that for large values of $B$, Algorithm \ref {a:UEUCB}    can be far from optimal. In this section,  we present our main scheme, which improves over this  shortcoming of Algorithm \ref{a:UEUCB}, and is only off by a minor additive term from our lower bound. Our scheme will combine an improved version of the $\UEUCB$ algorithm, $\UEUCBP$, and the $\CAS$ encoding algorithm.
 
 Like $\UEUCB$, $\UEUCBP$ is also a two-phase bandit algorithm, that performs uniform exploration in the first phase and runs the $\UCB$ algorithm in the second phase. The second phase of the algorithm proceeds in the same manner as that of the $\UEUCB$ algorithm, and  the improvement  over $\UEUCB$ is in the phase 1 of the algorithm. 
 
  \begin{figure}[t]
\centering
\begin{tikzpicture} [scale=1, every node/.style={scale=1}]
\node[draw,text width= 12cm, text height=,] {%
\begin{varwidth}{\linewidth}
            
            \algrenewcommand\algorithmicindent{0.7em}
\begin{algorithmic}[1]
\Statex  \hspace{-0.6cm} \textbf{Parameters:}  $L=2 \log B$, $\tau = \frac{2}{\SNR},$ $B_1^2=B^2$
\Statex \nonumber\For{$\ell \in [L]$} 
\Statex $\theta_{\ell}=\frac{\sqrt{P}}{B_{\ell}}$
\Statex \(B^2_{\ell+1}=\left(\frac{\frac{B^2_{\ell}}{\SNR}+1}{\tau}+1 \right)\)
\EndFor
\Statex \nonumber\For{$k \in [K]$} $\hat{\mu}_{0, k}=0$
\EndFor
  \Statex \hspace{-0.5cm}{\Large \color{blue}Phase $1$}

  \For{$\ell \in [L]$}
\For{ $k \in [K]$} 
 \For{ $t \in \{(\ell-1)K\tau +(k-1)\tau+1, \ldots, (\ell-1)K\tau+k\tau\}$}
 \State {\textbf{\color{red1}At Learner}}
 \Indent
\State   $A_t = k$ and $S_t=\hat{\mu}_{\ell-1, k}$
\State $A_t,$ $S_t$ to the client
\EndIndent
\State {\textbf{\color{teal}At Client}}
\Indent
\State  Play arm $A_t$ and observe $X_{t, A_t}$
\State Send $C_{\CAS}(X_{t, A_t}, S_t; \theta_{\ell})$ over the AWGN channel
\EndIndent
\State {\textbf{\color{red1}At Learner}}
 \Indent
 \State Observe AWGN Output $Y_t= C_{\CAS}(X_{t, A_t}, S_t; \theta_{\ell})+Z_t$
 \State Use $\tilde{X}_{t, A_t} = \frac{Y_t}{\theta_t}+S_t$ as the $t$th round reward
 \EndIndent
\EndFor
\State {\textbf{\color{red1}At Learner}}
 \Indent
 \State Form a mean estimate for arm $k$
 \[\hat{\mu}_{\ell, k}=\frac{1}{\tau}\cdot \sum_{t = (\ell-1)K\tau+(k-1)\tau+1}^{ (\ell-1)K\tau+k\tau} \tilde{X}_{t, k}\]
 \EndIndent

\EndFor
\EndFor

\Statex~
\Statex  \hspace{-0.6cm} \textbf{Parameters:} $\theta_{L+1}=\frac{\sqrt{P}}{B_{L+1}}$, $\eta_{L+1}=\frac{B^2_{L+1}}{\SNR}+1$
\Statex \hspace{-0.5cm}{\Large \color{blue}Phase $2$}
\For{$t \in \{K\tau+1, \ldots, T\}$}
\State{\textbf{\color{red1} At Learner}}
\Indent
\State   $\displaystyle{A_t = \arg\max_{k\in K} \UCB_k\left(t-1; \eta_{L+1}\right)}$ and $S_t=\hat{\mu}_{L, A_t}$
\State Send $A_t, S_t$ to Client
\EndIndent
\State{\textbf{\color{teal}At Client}} 
\Indent
\State Play arm $A_t$ and observe $X_{t, A_t}$
\State  Send $C_{\CAS}( X_{t, A_t}, S_t; \theta_{L+1})$ over the AWGN channel
\EndIndent

\State {\textbf{\color{red1}At Learner}} 
\Indent
\State Observe AWGN Output $Y_{t}=C_{\CAS}( X_{t, A_t}, S_t; \theta_{L+1})+Z_t$
\State Use $\tilde{X}_{t, A_t}=\frac{Y_t}{\theta_{L+1}}+S_t $ as the $t$th round reward 
\State Update UCB index for all arms
\EndIndent

\EndFor

\end{algorithmic}  
\end{varwidth}};
 \end{tikzpicture}

 \caption{$\UEUCBP$ Algorithm along with the $\CAS$ encoding algorithm}\label{a:UEUCBP}
\vspace{-0.5cm}  
 \end{figure}

\paragraph{Phase 1:}

 $\UEUCBP$, like $\UEUCB$, does uniform exploration in all rounds of phase 1. However, phase 1 is divided into $L$ different sub-phases, where the encoding of rewards is different in each sub-phase. The mean estimates formed in the previous sub-phase are used as side information in the current sub-phase and sent to the client for better encoding.

In more detail, we have $L=2 \log B$ sub-phases in total and each sub-phase runs for $K \tau$ rounds, where
\begin{align}\label{eq:tau}
\tau=\frac{2}{\SNR} \vee 2.
\end{align} Thus, phase 1 runs for $L \cdot K \tau \leq \frac{4K\log B}{\SNR}$ rounds.

\begin{rem}
Phase 1 in $\UEUCBP$ runs for $\oO{\big(\frac{K \log B}{\SNR}\big)}$ rounds as opposed to $\oO{\big(\frac{K B^2}{\SNR}\big)}$ rounds used in the  phase 1 of $\UEUCB$. We will show that despite performing uniform exploration for significantly fewer rounds for all arms, we will end up with mean estimates with the same variance as the means estimates formed after phase 1 of $\UEUCB$ because of a superior encoding protocol in this phase.
\end{rem}

The first sub-phase proceeds in exactly the same manner as phase 1 of  Algorithm \ref{a:UEUCB}. To describe the algorithm in the subsequent sub-phases, we set up some notation:  For all $k \in [K],$ denote by $\hat{\mu}_{\ell, k}$  the mean estimates formed for all arms at the end of the $\ell$th sub-phase. For a round $t$ in $\ell$th sub-phase, the learner informs the client about the arm $A_t$ it wants to play and the mean estimate formed for this arm in the $(\ell-1)$th sub-phase,  $\hat{\mu}_{\ell-1, A_t}.$

On the client's side, 
  the reward  corresponding to $A_t$ is encoded by using $\CAS$ with parameter $\theta_\ell=\frac{\sqrt{P}}{B_{\ell}}$ to get
 \[
 C_{\CAS}(X_{t, A_t}, \mu_{A_t}; \theta_\ell):= \frac{\sqrt{P}}{B_{\ell}} (X_{t, A_t}-\hat{\mu}_{\ell-1, A_t}).
 \]
 Here $B_\ell$, $\ell \in [L]$,  will be shown to be an upper bound on the $L_2$-norm of random variable $X_{t, A_t}-\hat{\mu}_{\ell-1, A_t}$ (see the proof of Theorem \ref{t:UEUCBP} ) and is  described by the recursion 
 \begin{align}\label{e:imp_rec}
B^2_{\ell+1}= \frac{\frac{B^2_{\ell}}{\SNR}+1}{\tau}+1   \quad \text{and} \quad B_1^2=B^2.
\end{align}

 The decoding at the learner is simply done by scaling $Y_t$ by $\frac{1}{\theta_\ell}$  and adding the mean estimate $\hat{\mu}_{\ell, A_t}$ to get
 \[ 
 \tilde{X}_{t, A_t}=\frac{B_{\ell}}{\sqrt{P}} Y_t+\hat{\mu}_{A_t}. 
 \]

\paragraph{Phase 2:} Phase 2 proceeds in the same manner as the phase 2 of Algorithm \ref{a:UEUCB}, with $\theta_{L+1}$ used for encoding at $\CAS$ and the parameter $\eta_{L+1}$ of the $\UCB$ algorithm  set to $\frac{B^2_{L+1}}{\SNR}+1.$

\begin{rem}
Equation~\eqref{e:imp_rec} sheds some light on the reasons behind $\UEUCBP$  requiring  fewer uniform exploration rounds relative to $\UEUCB$ to come up with mean estimates of similar variances. Essentially, the centered reward in the $\ell$th sub-phase has a much smaller $L_2$-norm than the centered reward in the $(\ell-1)$th sub-phase. This smaller second moment allows the client to scale the centered reward by a  larger factor, leading to  smaller variances for the decoded reward. This, in turn, leads to mean estimates formed in that sub-phase to have smaller variances or, in other words, are more accurate. Then, in the $(\ell+1)$th sub-phase, since we center the rewards around these more accurate mean estimates, the centered rewards have even smaller $L_2$ norms. This positive reinforcement cycle leads to the variance of the mean estimates decreasing exponentially with respect to sub-phases, which, in turn, leads to much fewer uniform exploration rounds to come up with mean estimates with the same variance as those after phase 1 of $\UEUCB.$
\end{rem}

$\UEUCBP$, along with its interaction with $\CAS$, is described in Algorithm \ref{a:UEUCBP}.

\begin{thm}\label{t:UEUCBP}
For all bandit instances $\nu \in \cE$, the $\UEUCBP$ bandit algorithm  along with $\CAS$ encoding algorithm  as described in Algorithm \eqref{a:UEUCBP} satisfies the power constraint \eqref{eq:Power_Constraint} at all rounds and has the following regret guarantee:
\eq{R_T(\nu, \UEUCBP, \CAS)  
\leq  \frac{8 K B \log B}{\SNR \wedge 1}+6KB + 8 \sqrt{\left(\frac{4}{\SNR}+1\right)} \cdot \sqrt{ KT \ln T}.}
\end{thm}
The key to the proof of Theorem \ref{t:UEUCBP} is to show that the second moment of the centered rewards $X_{t, A_t}-\hat{\mu}_{\ell-1, A_t}$ is less than $B_\ell,$ $\ell \in [L],$ when round $t$ belongs to the $\ell$th sub-phase of phase 1, and when round $t$ belongs to phase 2, it is less than $B_{L+1}.$ Similar arguments as in the proof of Theorem \ref{t:UEUCB} can show this. We defer the details to the Appendix.

\begin{rem}
The $\ln T$ factor in the third term of Theorem \ref{t:UEUCBP} can be removed by using {\tt MOSS} \cite{audibert2009minimax} instead of the $\UCB$ algorithm. Our work focuses on understanding how the regret grows as a function of $\SNR$; we are not as concerned with logarithmic factors in $T.$ Hence, we chose the more popular $\UCB$ algorithm.
\end{rem}



 
 
 
\section{Concluding Remarks}
 The near-optimal Algorithm \ref{a:UEUCBP} use a highly adaptive encoding protocol, where the history of the game was used to refine the encoding in a particular round. On the other hand, \cite{acharya2021information} showed for distributed convex, Lipschitz optimization under information constraints, a fixed, non-adaptive encoding protocol gives optimal performance. This suggests an interesting distinction between encoding for distributed sequential decision-making problems and distributed optimization problems. 
 
 An interesting open problem is closing the additive $\frac{8 K B \log B}{\SNR \wedge 1}$ gap between our upper and lower bound. Furthermore, the refinement to our model where the learner-to-client communication is also over an AWGN channel and the multi-client setting are interesting future research directions.
 
 \section*{Acknowledgement}
 This work is supported by the Singapore National Research Foundation (NRF) Fellowship programme (Grant Numbers: A-0005077-01-00 and A-0008064-00-00) and an MOE Tier 1 Grant (Grant Number: A-0009042-01-00).
 

 \bibliographystyle{IEEEtranS}
\appendix
\onecolumn

\section{Proof of Theorem \ref{t:UEUCBP} (Regret bound of $\UEUCBP$)}
We will first show that the centered rewards transmitted by the client satisfy the power constraint. Note that the centered reward coded by the client at any round in the $\ell$th sub-phase of phase  1 is $X_{t, A_t} -\hat{\mu}_{\ell-1, A_t}$. Furthermore, the client scales the reward by a factor of $\frac{\sqrt{P}}{B_{\ell}}$. Thus, we need to show that $\E{(X_{t, A_t} -\hat{\mu}_{\ell-1, A_t})^2} \leq B^2_\ell $.
Towards proving this, we write
\eq{&
\E{(X_{t, A_t} -\hat{\mu}_{\ell-1, A_t})^2} \\&=  \E{(X_{t, A_t} -\mu_{A_t})^2} + \E{(\hat{\mu}_{\ell-1, A_t}-\mu_{A_t})^2}\\
&\leq 1+ \E{(\hat{\mu}_{\ell-1, A_t}-\mu_{A_t})^2}.
}
From the description of Algorithm \ref{a:UEUCBP}, for $k \in [K],$ we have
\eq{\hat{\mu}_{\ell-1,k}= \frac{1}{\tau}\sum_{t= (\ell-2)K\tau + (k-1)\tau +1}^{ (\ell-2)K\tau+k \tau}\left( X_{t,k}+ \frac{B_{\ell-1}}{ \sqrt{P}}Z_t\right).}
Therefore, we have
\eq{
 \E{(\hat{\mu}_{\ell, A_t}-\mu_{A_t})^2}  \leq \frac{\frac{B_{\ell-1}^2}{\SNR}+1}{\tau}.
}
Using the preceding inequalities and definition of $B_{\ell}$ in \eqref{e:imp_rec}, we get
\eq{
\E{(X_{t, A_t} -\mu_{\ell-1, A_t})^2} \leq 1+\frac{\frac{B_{\ell-1}^2}{\SNR}+1}{\tau}=B_{\ell}^2.
}
To prove that the power constraint is satisfied in the second phase, note  that the centered reward transmitted by the client in the second phase is $ X_{t, A_t} -\hat{\mu}_{L, A_t}.$  Furthermore, client scales the reward by a factor of $\frac{\sqrt{P}}{B_{L+1}}.$ By precisely the same arguments as above, we can deduce that $\E{(X_{t, A_t} -\hat{\mu}_{L, A_t})^2}\leq B_{L+1}^2$.

To bound the regret, first note that we incur linear regret in phase 1 of $\UEUCBP$. Specifically, since the time horizon for phase 1 is $\frac{4 K\log B}{\SNR \wedge 1}$ and each round's regret is at the most $2B$, we incur  a total regret of $\frac{8 K B\log B}{\SNR \wedge 1}$ in the first phase.

To bound the regret in the second phase, note that the centered reward estimate $\tilde{X}_{t, A_t} -\mu_{A_t}$ in the second phase is subgaussian with a variance factor of $\frac{B^2_{L+1}}{\SNR}+1.$ Substituting the value of $\tau$ from ~\eqref{eq:tau}~ in \eqref{e:imp_rec}, we obtain
\eq{
B^2_{\ell+1}&=\frac{\frac{B^2_{\ell}}{\SNR} +1}{\tau}+1\\
&\leq \frac{B^2_{\ell}}{2} +\frac{\SNR \wedge 1}{2}+1.
}
Rearranging the terms, we have
\eq{
B^2_{\ell+1} -2\left(\frac{\SNR  \wedge 1}{2}+1\right) \leq \frac{B^2_{\ell} -2\left(\frac{\SNR  \wedge 1}{2}+1\right) }{2}.
}
Recursively applying the above identity, it follows that
\eq{
B^2_{L+1} -2\left(\frac{\SNR \wedge 1}{2}+1\right) \leq \frac{B^2 -2\left(\frac{\SNR \wedge 1}{2}+1\right) }{2^L}.
}
Substituting $L=2 \log B,$ we get $B^2_{L+1}$ less than $4$. 

The proof is then completed by using Lemma \ref{l:UCB} to bound the regret in the second phase and adding the regret from the first phase.

\section{Proof of Theorem \ref{t:lb} (Lower Bound)}\label{s:ProofLB}

We will first show that for $T \geq c_2 K$ and a universal constant $c_1 \in (0, 1),$ we have 
\[R_T^\ast(\SNR) \geq c_1KB. \]

From the definition of $R_T^\ast(\SNR) $ in~\eqref{e:minmax_regret}, it is enough to prove a lower bound on the average regret when the problem instances are uniformly chosen from a subset of $\mathcal{E}.$
We choose that finite subset to be the instances where a single arm has deterministic reward $B$, and the others deterministic $-B$. Therefore, we have $K$ different instances with the probability of each  instance occurring being $1/K.$ Now, observe that regardless of the algorithm, the probability of finding the $B$-reward arm is at the most $\frac{1}{2}$ in the first $\floor{K/2}$ arm pulls. Hence the regret in the first $\floor{K/2}$ arm pulls when averaged over the $K$ instances must be at   least $\frac{1}{2} \cdot \floor{K/2} \cdot 2B = c_1
KB. $

In the rest of proof we will show that for  $T \geq  \frac{c_2 K}{\frac{1}{2} \log (1+ \SNR) \wedge 1},$ we have
\begin{align}\label{e:main_lb}
R^\ast_T(\SNR) \geq  c_1 \left(\sqrt{KT} \cdot \frac{1}{ \sqrt{\SNR  \wedge 1}} \right).
\end{align}
Then, noting that both the bounds hold for $T \geq  \frac{c_2 K}{\frac{1}{2} \log (1+ \SNR) \wedge 1}$ and using the fact that $a \vee b \geq \frac{a+b}{2}$ completes the proof.



 Consider the bandit instance $\nu$ where rewards for each arm pull are either $-1$ or $1$ and the arms have the following means: $\mu_1=\Delta$, and $\mu_i=0$ for all $i \in \{2, \ldots, K\},$ where $\Delta$  is some parameter in $(0, 1/4)$ whose precise value we will set later. Using the fact that rewards for all arms lie in $[-1, 1]$ and by Hoeffding's lemma \cite[Lemma 2.2]{boucheron2013concentration}, it is easy to see that the centered rewards for all arms are subgaussian with a variance factor of 1. Therefore, it follows that   $\nu \in \mathcal{E}.$  Now, fix a bandit algorithm $\pi \in \Pi_T$ and an encoding scheme $C^T \in \mathcal{C}^T.$ Denote by $P^T$ the probability distribution of the sequence $( A_1, Y_1, A_2, Y_2, \ldots,  A_T, Y_T)$ induced by the algorithm $\pi$ and encoding strategy $C^T$ on the instance $\nu$.  
 
 Denote by $N_i(T)$ the number of times arm $i$ has been played up to and including  round $T$. Let 
 \[j = \argmin_{i > 1} \EE{P^T}{N_i(T)}.\]
 That is, $j$ is the arm pulled least amount of times from the arms $\{2, \ldots, K\}.$ Trivially,
 \begin{align}\label{e:npulls}
      \EE{P^T}{N_i(T)} \leq \frac{T}{K-1}.
 \end{align}
 
 Now consider another bandit instance $\nu^{\prime}$ where again the rewards are either $-1$ or $1$ with the following means: $\mu_1=\Delta$, $\mu_j=2 \Delta$, $\mu_i=0$ for all $i \in \{2, \ldots, K\}\setminus\{j\}.$ Let $Q^T$ be the  the probability distribution of the sequence $(A_1, Y_1,  A_2,  Y_2, \ldots, A_T, Y_T)$ induced by the algorithm $\pi$ and encoding strategy $C$ on the instance $\nu^{\prime}$.
 
 By  a simple calculation, we have
 \[R_T(\nu, \pi, C) \geq \frac{T\Delta}{2}  P^T(N_1(T) \leq T/2)  \quad  \text{and} \quad  R_T(\nu^{\prime}, \pi, C) \geq \frac{T\Delta}{2} Q^T(N_1(T) > T/2).\]
 
Combining the two inequalities, using the variational formula for total variation distance, and applying Pinsker's inequality, we get
 \begin{align} \nonumber
     R_T(\nu, \pi, C) + R_T(\nu^{\prime}, \pi, C)
& \geq \frac{T\Delta}{2} \left(  P^T(N_1(T) \leq T/2)  + Q^T(N_1(T) > T/2)\right)\\ \nonumber
 &= \frac{T\Delta}{2} \left( 1  - (P^T(N_1(T) > T/2) -Q^T(N_1(T) > T/2) )\right)\\ \nonumber
 & \geq \frac{T\Delta}{2} \left( 1- \mathrm{d_{TV}}(P^T, Q^T)\right)\\ \label{eq:REGRETKLT}
 & \geq \frac{T\Delta}{2} \left( 1- \sqrt{\frac{1}{2}D(P^T\mid\mid Q^T)}\right).
 \end{align}
 
 
Note that $P^T$ and $Q^T$ differ only when arm $j$ is played. 
Recall that the dependence of $Y$ on an encoding strategy $C$ and reward $X$ is given by \eqref{eq:Channel_Output}. For encoding scheme $C$, used in a round $t$, denote by $C( \cdot \mid x)$ the conditional distribution of $Y$ given that the  reward is $X=x$. Let $p_j\circ C$ and $q_j\circ C$ denote the distribution of $Y$ when arm $j$ is played at round $t$. Recalling the reward distribution for arm $j$, it follows that,
\begin{align*}
    (p_j \circ C) (Y=y) &= \frac{ C(y \mid 1)}{2}+ \frac{C(y\mid-1)}{2} \quad \text{and}\\  
    (q_j\circ C) (Y=y) &= \frac{ C(y \mid 1)}{2}+ \frac{C(y\mid-1)}{2} +\Delta (C(y \mid 1) -C(y\mid-1) ).
\end{align*}

Moreover, using similar arguments as in  \cite[Lemma 15.1]{lattimore2020bandit}, we can prove the following lemma. 
\begin{lem}\label{e:KL_ineq}
\[D(P^T\mid\mid Q^T) \leq \EE{P^T}{N_j(T)} \cdot \sup_{C \in \C} D(p_j \circ C \mid\mid q_j \circ C )\]
\end{lem}

We defer the proof the Lemma to Appendix \ref{C} and proceed with proving the lower bound. Combining Lemma \ref{e:KL_ineq} with \eqref{e:npulls}, we have
\eq{
D(P^T\mid\mid Q^T) 
&\leq \frac{T}{K-1} \sup_{C \in \C} D(p_j \circ C \mid\mid q_j \circ C ).
}

We will now use the following fact.
\begin{fact}\label{fact1}
Recall that $d_{\chi^2}(p, q) \coloneqq \int \frac{(p(x)-q(x))^2}{q(x)} \mathrm{d}x$ denotes the $\chi^2$ distance between probability measures $p$ and $q$ when they are absolutely continuous with respect to the Lebesgue measure $\mathrm{d}x$. Then, it holds that \[D(p \mid \mid q) \leq d_{\chi^2}(p, q).\]
\end{fact}
Fact \ref{fact1} is well known (see, for instance \cite{csiszar2004information}]), but we provide a short proof in Appendix \ref{C} for completeness.

Using this fact, $D(p_j \circ C \mid\mid q_j \circ C )$ can be bounded as follows:
\begin{align*}
\mathrm{D}(p_i \circ C \mid\mid q_i \circ C) &\leq 
d_{\chi^2}(p_i \circ C , q_i \circ C)\\
&= \int \frac{(p_i \circ C -q_i \circ C)^2}{q_i \circ C} \, \mathrm{d}y\\
&=  \Delta^2 \int \frac{(C(y\mid 1) -C(y \mid -1))^2}{ \frac{ C(y \mid 1)}{2}+ \frac{C(y\mid-1)}{2} +\Delta (C(y \mid 1) -C(y\mid-1) )} \, \mathrm{d}y
\end{align*}Now note that for $\Delta \leq 1/4,$ we have 
$\frac{C(y \mid 1)}{2}+ \frac{C(y\mid-1)}{2} +\Delta (C(y \mid 1) -C(y\mid-1) ) \geq \frac{C(y \mid 1)}{4}+ \frac{C(y\mid-1)}{4}.$ Using this fact, we get 
\eq{
 \mathrm{D}(p_i \circ C \mid\mid q_i \circ C)&\leq 2\Delta^2 \int \frac{(C(y\mid 1) -C(y \mid -1))^2}{ \frac{ C(y \mid 1)}{2}+ \frac{C(y\mid-1)}{2}} \, \mathrm{d}y\\
&= 2\Delta^2 \left ( \frac{1}{2}d_{\chi^2}(C( \cdot \mid 1), C( \cdot \mid 1)/2 + C( \cdot \mid -1)/2) \right. \\ &\hspace{3cm} \left.+\frac{1}{2} d_{\chi^2}(C( \cdot \mid -1), C( \cdot \mid 1)/2 + C( \cdot \mid -1)/2)\right)\\
&\leq 4\Delta^2 \left ( \frac{1}{2}\mathrm{D}(C( \cdot \mid 1) \mid \mid C( \cdot \mid 1)/2 + C( \cdot \mid -1)/2) \right. \\ &\hspace{3cm} \left.+\frac{1}{2} \mathrm{D}(C( \cdot \mid -1) \mid \mid C( \cdot \mid 1)/2 + C( \cdot \mid -1)/2)\right),
}
where the equality is easily verified by substituting the definition $d_{\chi^2}$ and the  last inequality uses the following fact. 

To proceed further, we will use the following fact.
\begin{fact}\label{fact2}
If $\frac{\mathrm{d}P}{\mathrm{d}Q} (x) \leq c \quad\forall x\in \R$, then
\[d_{\chi^2}(P, Q) \leq c D(P \mid \mid Q).\]
\end{fact}
Fact 2 is also known \cite{dragomir2000upper}, but we provide a proof in Appendix \ref{C} for completeness.

Denote by $V$ a hypothetical reward  that takes values $-1$ and $1$ with equal probability, and by $Y$ the output of the AWGN channel  when $V$ is further  coded to $C(V)$ to satisfy \eqref{eq:Power_Constraint} and transmitted over the AWGN channel. Then, \[\frac{1}{2}\mathrm{D}(C( \cdot \mid 1) \mid \mid C( \cdot \mid 1)/2 + C( \cdot \mid -1)/2) +\frac{1}{2} \mathrm{D}(C( \cdot \mid -1) \mid \mid C( \cdot \mid 1)/2 + C( \cdot \mid -1)/2) \] is mutual information between $V$ and $Y$, denoted by $I(V; Y).$  To see this, note that $I(W_1; W_2) \coloneqq D( P_{W_1, W_2} \mid \mid P_{W_1} \times P_{W_2}) = \EE{P_{W_1}}{ D( P_{ W_2|W_1} \mid \mid  P_{W_2}) }.$


We have therefore shown that 
\begin{align}\label{eq:Imp_LowerBound_Eq}
    D(P^T\mid\mid Q^T) \leq  \frac{4\Delta^2T}{K-1}\sup_{C \in \C}I(V; Y).
\end{align}

Note that $I(V; Y) \leq H(V) \leq \ln 2.$ Since $V-C(V)-Y$ forms a Markov chain in this order, by the data processing inequality \cite[Chapter 2]{CovTho06}, we have that 
\eq{
I(V; Y) \leq I (C(V); Y).
}
Then, using the fact that $I (C(V); Y)$ can be further upper bounded by the capacity of an AWGN channel~\cite[Chapter 9]{CovTho06}, we have
\eq{
I (C(V); Y) \leq \frac{1}{2}\ln(1+\SNR)=\frac{1}{2} \log (1+\SNR) \cdot \ln 2.
}
Thus, we have shown that 
\eq{ D(P^T\mid\mid Q^T)  \leq   \frac{4 \ln 2 \cdot \Delta^2T }{K-1}\left(\frac{1}{2}\log (1+\SNR) \wedge 1 \right).
 }
 
 Substituting this into \eqref{eq:REGRETKLT}, we get
 \eq{
  R_T(\nu, \pi, C) + R_T(\nu^{\prime}, \pi, C) \geq \frac{T\Delta}{2}\left( 1 - \sqrt{\frac{4 \ln 2 \cdot \Delta^2 T}{K-1} \left(\frac{1}{2}\log(1+\SNR) \wedge 1\right)}\right).
 }
Setting $\Delta^2=\frac{K-1}{16 \ln 2  \cdot T \left( \frac{1}{2}\log(1+\SNR) \wedge 1 \right)}$ completes the proof. (Note that we can only set such a $\Delta$ and ensure that $\Delta \leq \frac{1}{4}$, if $T \geq \frac{K-1}{\ln 2 \left( \frac{1}{2}\log(1+\SNR) \wedge 1 \right)},$  which is the reason we need this condition in the statement of our lower bound.)

We will now  show that 
 for  $T \geq \frac{K-1}{\ln 2 \left( \frac{1}{2}\log(1+\SNR) \wedge 1 \right)} \geq K$  and for a universal constant $c,$ we have
 \[R_T^\ast(\SNR) \geq c KB.\] 
 In fact, we show the stronger bound
 \[ R_K^\ast(\SNR) \geq c KB.  \]
Then, using the fact that $a \vee b \geq \frac{a+b}{2}$ completes the proof.

We now choose our instance $\nu$ where rewards for each arm pull are Gaussian with variance $1$ and the following means: $ \mu_1=1, $ $\mu_i=0 $ for all $i \in \lbrace 2, \ldots,  K\rbrace.$ Once again, let $j$ be the arm pulled least amount of times. Then define the alternate instance $\nu^{\prime}$ where rewards for each arm pull are Gaussian with variance $1$ and the following means: $\mu_1=B/2, \mu_j=B-1, \mu_i=0 $ for all $i \in \lbrace 2, \ldots,  K\rbrace/{j}.$ All the arm rewards have Gaussian distribution with the variance 1 and means stated above. Proceeding as earlier, we have
 \begin{align} \nonumber
     R_T(\nu, \pi, C) + R_T(\nu^{\prime}, \pi, C)
& \geq   \frac{T\Delta}{2} \left( 1- \mathrm{d_{TV}}(P^T, Q^T)\right)
\end{align}

Using the fact that $T \geq \frac{K-1}{\ln 2 \left( \frac{1}{2}\log(1+\SNR) \wedge 1 \right)} \geq c (K-1)$ and the fact that for our chosen distributions $1- \mathrm{d_{TV}}(P^T, Q^T)$ is a constant completes the proof.

\begin{rem}{
({\em Communication-constrained setting of \cite{hanna2022solving}})}
We now remark on possibility of recovering the lower bound in \cite{hanna2022solving} with our proof technique, albeit with weaker constants. Note that
\eqref{eq:REGRETKLT} also holds for the setting considered in \cite{hanna2022solving}, where the output $Y$ is restricted to $r$ bits. Therefore, $I(V; Y) \leq H(Y) \wedge H(V) =(\ln 2)(r \wedge 1)$. Proceeding, in the rest of the proof, {\em mutatis mutandis}, we get that the regret in this setting is bounded by 
$\Omega(\sqrt{KT}\cdot \frac{1}{r \wedge 1}).$ Such a lower bound suggests that even when using code of length 1 to encode our rewards, it might be possible to achieve the regret of the centralized setting. This is indeed shown to be possible by the scheme in \cite{hanna2022solving}. 
\end{rem}

 \section{Proofs of Intermediary Results in the Proof of Theorem \ref{t:lb}}\label{C}
 
 \subsection*{Proof of Lemma \ref{e:KL_ineq}}
 We prove the inequality by similar arguments as in  \cite[Lemma 15.1]{lattimore2020bandit}.
 
Recall that $H_t=\lbrace A_1, Y_1, \ldots, A_{t}, Y_{t}\rbrace$ denotes the history of the game till time $t$. Also recall that the hint $S_t$ is a function of history of the game till time $t-1$. Therefore the encoding scheme  at time $t$ depends on $H_{t-1}$, so we write it as $C_{H_{t-1}}.$ We additionally denote by $p_{A_t}  $ and $q_{A_t}$ the distribution of reward $X_{t, A_t}$ under instances $\nu$ and $\nu^{\prime},$ respectively.  Therefore, $p_{A_t} \circ C_{H_{t-1}}$ and $q_{A_t} \circ C_{H_{t-1}} $ denote distribution of $Y_t$ for an arm pull $A_t$ and history $H_{t-1}$ under instances $\nu$ and $\nu^{\prime},$ respectively. 

Using the definition of KL divergence, we have

\begin{align}\label{e:KL_mid}
\nonumber
D(P^T\mid\mid Q^T) &=  \mathbb{E}_{P^T} \left[\log \left(\frac{P^T( A_1, Y_1, \ldots, A_T, Y_T)}{ Q^T(A_1, Y_1, \ldots,  A_T, Y_T)}\right) \right]\\  
&= \sum_{t=1}^{T} \mathbb{E}_{P^T} \left[ \log \frac{ P^T( A_t, Y_t \mid  H_{t-1}) }{ Q^T(A_t, Y_t \mid H_{t-1})}\right].
\end{align}
By the definition of the bandit algorithm $\pi,$  the distribution over all actions to be  played remains the same for different problem instances as long as the history is the same. Therefore,
\eq{
\frac{ P^T( A_t, Y_t \mid  H_{t-1}) }{ Q^T(A_t, Y_t \mid H_{t-1})} &= \frac{ P^T( A_t\mid  H_{t-1}) }{ Q^T(A_t \mid H_{t-1})} \cdot \frac{ P^T( Y_t \mid  A_t, H_{t-1}) }{ Q^T(Y_t \mid A_t, H_{t-1}) }\\
&=\frac{ P^T( Y_t \mid  A_t, H_{t-1}) }{ Q^T(Y_t \mid A_t, H_{t-1}) }.
}

Substituting the above identity in \eqref{e:KL_mid}, we have
\begin{align*}
D(P^T\mid\mid Q^T) 
&= \sum_{t=1}^{T} \mathbb{E}_{P^T} \left[ \log \frac{ P^T( Y_t \mid A_t,  H_{t-1}) }{ Q^T(Y_t \mid A_t, H_{t-1})} \right]\\
&= \sum_{t=1}^{T} \mathbb{E}_{P^T} \left[ \mathbb{E}_{P^T} \left[ \left. \log \frac{ P^T( Y_t \mid A_t, H_{t-1}) }{ Q^T(Y_t \mid  A_t,  H_{t-1})}  \right| A_t, H_{t-1} \right]\right]\\
&= \sum_{t=1}^{T} \mathbb{E}_{P^T} \left[  D(p_{A_t}\circ  C_{H_{t-1}} \mid\mid q_{A_t}\circ  C_{H_{t-1}})\right],
\end{align*}
where the second last inequality uses the law of iterated expectations and the final one uses the definition of KL divergence.

Since $p_{A_t}$ and $q_{A_t}$ differ only when arm $A_t=j,$ we have
\begin{align*}
D(P^T\mid\mid Q^T) 
&=   \sum_{t=1}^{T} \mathbb{E}_{P^T} \left[ \mathbf{1}_{A_t=j}D( p_j \circ C_{H_{t-1}}\mid \mid  q_j \circ C_{H_{t-1}})\right]\\
&\leq  \sum_{t=1}^{T} \mathbb{E}_{P^T} \left[ \mathbf{1}_{A_t=j} \sup_{C \in \C}D( p_j \circ C\mid \mid  q_j \circ C)\right]\\
&\leq  \mathbb{E}_{P^T}{[N_j(T)]} \cdot \sup_{C \in \C} D(p_j \circ C \mid\mid q_j \circ C )\\
&\leq \frac{T}{K-1} \sup_{C \in \C} D(p_j \circ C \mid\mid q_j \circ C ).
\end{align*}

\qed

\subsection{Proof of Fact \ref{fact1}}
Using the fact that $\ln x \leq x -1$ and $\int\big(p(x)-q(x)\big)\mathrm{d}x=0$ , we have
\eq{D(p \mid \mid q) &=\int p(x) \ln \frac{p(x)}{q(x)} \mathrm{d}x - \int\left(p(x)-q(x)\right)\mathrm{d}x\\
&\leq \int p(x) \left(\frac{p(x)}{q(x)}-1\right) \mathrm{d}x-\int\left(p(x)-q(x)\right)\mathrm{d}x\\
&= \int  \frac{(p(x)-q(x))^2}{q(x)}\mathrm{d}x\\
&= d_{\chi^2}(p, q). 
}

\qed

\subsection{Proof of Fact \ref{fact2}}
We have
\eq{
d_{\chi^2}(p, q)&= \int \frac{(p(x) -q(x))^2}{q(x)}  \mathrm{d}x \\
&= \int p(x) \left(\frac{p(x)}{q(x)}-1\right) \mathrm{d}x-\int\left(p(x)-q(x)\right)\mathrm{d}x\\
&\leq   \int p(x) \cdot \frac{p(x)}{q(x)} \cdot \ln \frac{p(x)}{q(x)} \mathrm{d}x\\
&\leq  c \int p(x)  \ln \frac{p(x)}{q(x)} \mathrm{d}x
= c D(p \mid \mid q),
}
where the fist inequality follows by using the fact that $\frac{x-1}{x}\leq \ln x \quad \forall x \geq 0$ and $\int\left(p(x)-q(x)\right)\mathrm{d}x=0$ and the second inequality follows from the fact that $\frac{\mathrm{d}P}{\mathrm{d}Q}(x) \leq c  \quad\forall x\in \R.$

\qed


\end{document}